\documentclass[conference]{IEEEtran}
\IEEEoverridecommandlockouts
\usepackage{cite}
\usepackage{amsthm}
\usepackage{amssymb}
\usepackage{amsfonts}
\usepackage{amsmath}
\newtheorem{theorem}{Theorem}
\newtheorem{definition}{Definition}

\usepackage{algorithmic}
\usepackage{graphicx}
\usepackage{enumerate}
\usepackage[inline]{enumitem}
\usepackage{graphicx}
\usepackage{subfigure} 
\usepackage{threeparttable}
\usepackage{cleveref}
\usepackage[ruled,vlined,linesnumbered]{algorithm2e}
\usepackage{footmisc}
\usepackage{url}  
\usepackage{subfigure}
\usepackage{booktabs}
\usepackage{color}
\def\BibTeX{{\rm B\kern-.05em{\sc i\kern-.025em b}\kern-.08em
    T\kern-.1667em\lower.7ex\hbox{E}\kern-.125emX}}
\begin{document}

\title{Regularized Non-negative Spectral Embedding for Clustering}

\author{\IEEEauthorblockN{Yifei Wang, Rui Liu, Yong Chen(Corresponding), Hui Zhang, Zhiwen Ye}
\IEEEauthorblockA{School of Computer Science and Engineering, Beihang University, Beijing 100191, China \\
\{yifeiwang,lr\}@buaa.edu.cn, \{chenyong,hzhang\}@nlsde.buaa.edu.cn, yezhiwen@buaa.edu.cn}
}

\maketitle

\begin{abstract}
Spectral Clustering is a popular technique to split data points into groups, especially for complex datasets.
The algorithms in the Spectral Clustering family typically consist of multiple separate stages (such as similarity matrix construction, low-dimensional embedding, and K-Means clustering as post-processing), which may lead to sub-optimal results because of the possible mismatch between different stages.
In this paper, we propose an end-to-end single-stage learning method to clustering called Regularized Non-negative Spectral Embedding (RNSE) which extends Spectral Clustering with the adaptive learning of similarity matrix and meanwhile utilizes non-negative constraints to facilitate one-step clustering (directly from data points to clustering labels).
Two well-founded methods, \emph{successive alternating projection} and \emph{strategic multiplicative update}, are employed to work out the quite challenging optimization problems in RNSE.
Extensive experiments on both synthetic and real-world datasets demonstrate RNSE's superior clustering performance to some state-of-the-art competitors.
\end{abstract}

\begin{IEEEkeywords}
Spectral Clustering, One-stage Learning, Successive Alternating Projection, Strategic Multiplicative Update
\end{IEEEkeywords}
\section{Introduction}
Clustering is an important unsupervised learning task which aims to group a set of data objects into clusters in such a way that objects in the same cluster are more similar to each other than those in different clusters.
For complex datasets, Spectral Clustering \cite{Andrew-NIPS-SpectralClustering-2002} and its many variants\cite{Zhang-RoSC-2019,Shaham-SNet-2018,Jin-deepSC-2017} are particularly popular due to their ability of discovering highly non-convex clusters.
Such algorithms make use of the spectrum of the similarity matrix of the data to perform dimensionality reduction before grouping objects in a low-dimensional space \cite{Andrew-NIPS-SpectralClustering-2002,Mikhail-Eigenvalues-NIPS-2002,Ron-NIPS-DSN-SpectralClustering-2007,Deng-SR-ICDM-2007,Xinlei-AAAI-LSC-2011,Feiping-AAAI-CLR-2016,Feiping-AAAI-ULGE-2017}.
Typically, the algorithms in the Spectral Clustering family consist of multiple separate stages as follows:
\begin{enumerate}[leftmargin=*,label={(\arabic*)}] 
  \item Construct a pairwise similarity matrix, e.g., according to the $k$-nearest-neighbors graph of the data;
  \item Compute the corresponding Laplacian matrix and normalize it;
  \item Represent the data objects in a low-dimensional space using a few eigenvectors of the Laplacian matrix corresponding to its smallest eigenvalues;
  \item Re-normalize the embedded data vectors and then group them into $C$ clusters, e.g., through the classic K-Means algorithm \cite{Stuart-IEEE-K-means-1982}.
\end{enumerate}
The above multi-stage approach may lead to sub-optimal clustering results due to the possible mismatch between different stages.
Moreover, there is still much room for improvement in the optimization methods.
For example, normalizing the similarity matrix into a doubly stochastic matrix\footnote{A doubly stochastic matrix is a square matrix that satisfies $\mathbf{S}\mathbf{1}=\mathbf{1}$, $\mathbf{S}^T\mathbf{1}=\mathbf{1}$, and $\mathbf{S} \ge 0$, where $\mathbf{1}$ is a column vector with all elements to be $1$ and $\mathbf{S} \ge 0$ represents element-wise non-negativity.} has been found to be beneficial \cite{Ron-NIPS-DSN-SpectralClustering-2007}, and imposing some global priors (like Laplacian rank) could help to reveal the underlying clustering structure of the dataset \cite{Feiping-AAAI-CLR-2016,Lefei-IJCAI-AMRMF-2017}.

In light of the above analysis, we extend Spectral Clustering and propose an end-to-end single-stage learning framework for clustering named ``Regularized Non-negative Spectral Embedding'', or RNSE in short.
It does not rely on a predefined similarity matrix but learn the similarity matrix in a data-driven self-adaptive manner.
Furthermore, it introduces two global priors, i.e., the doubly stochastic constraint (for normalizing the similarity matrix) and the non-negative low-rank constraint (for capturing the intrinsic clustering structure), to facilitate the optimization.
The effectiveness of RNSE for clustering has been confirmed by extensive experiments on both synthetic and real-world datasets.
\section{Related Work}
Our proposed clustering technique RNSE is closely connected to several research areas.

First of all, RNSE obviously has its roots in the Spectral Clustering \cite{Andrew-NIPS-SpectralClustering-2002} and also its many variants including Laplacian Eigenmap (LE) \cite{Mikhail-Eigenvalues-NIPS-2002}, Locality Preserving Projections (LPP) \cite{Xiaofei-LPP-NIPS-2004}, and Spectral Regression (SR) \cite{Deng-SR-ICDM-2007}. 
Those techniques start from a predefined pairwise similarity matrix and perform clustering (and other tasks) via spectral decomposition.
They typically consist of several separate stages.
In contrast, our proposed RNSE technique carries out the whole process from the given data straight to the clustering result in just one stage, with the similarity matrix automatically learned.

Next, the formulation of RNSE imposes non-negative constraints, so it is related to the series of Non-negative Matrix Factorization (NMF) \cite{Daniel-NMF-learning-objects-Nature-1999} techniques including Non-negative Matrix Tri-factorization (NM3F) \cite{Huifang-PAKDD-NM3F-2010} and Graph regularized NMF (GNMF) \cite{Deng-GNMF-TPAMI-2011}.
Those techniques decompose the data matrix into two or more low-rank non-negative matrices from which the clustering structures of the data could be read out.
However, our proposed RNSE technique is different from them as it contains two sub-problems of optimization with non-negative constraints which are combined in a unified optimization framework.
Thus RNSE could obtain the clustering results directly after solving the optimization problem, while those NMF-based methods need some post-processing such as using the K-Means algorithm \cite{Stuart-IEEE-K-means-1982} to get the final clustering results.

Last but not the least, RNSE utilizes structural regularization in its learning algorithm.
Generally speaking, it is useful to incorporate appropriate priors into the learning process as the priors could help to find the intrinsic structure of data.
Locality Preserving Projections (LPP) \cite{Xiaofei-LPP-NIPS-2004} tries to maintain the $k$-nearest-neighbors graph while performing linear dimensionality reduction of the data.
Similarly, Graph Regularized Non-negative Matrix Factorization (GNMF) \cite{Deng-GNMF-TPAMI-2011} adds a $k$-nearest-neighbors graph based regularizer term to the vanilla NMF algorithm.
Those two techniques both use the local (manifold) structure of data for regularization.
There also exist techniques with global structural regularization.
For example, Doubly Stochastic Normalization (DSN) \cite{Ron-NIPS-DSN-SpectralClustering-2007} enforces the doubly stochastic condition on the similarity matrix before carrying out Spectral Clustering.
Besides, Structured Doubly Stochastic Matrix (SDS) \cite{Xiaoqian-KDD-2016}, Clustering with Adaptive Neighbors (CAN) \cite{Feiping-KDD-CAN-2014} and Constrained Laplacian Rank (CLR) \cite{Feiping-AAAI-CLR-2016} borrow the idea of $C$-connected components (cf. Theorem~\ref{C_connected_clusters}) from the spectral graph theory to form the regularization for their learning algorithms.
Inspired by the above methods, our proposed RNSE technique utilizes both global structures (i.e., doubly stochastic matrix and non-negativity constrained $C$-connected components) as regularizers for clustering.
\section{The Proposed Approach}
Given a data matrix $\mathbf{X} = [{x_1},{x_2}, \cdots ,{x_N}] \in {\mathbb{R}^{M \times N}}$, where $M$ is the dimension of a sample, $N$ marks the number of the total samples, and $x_i \in \mathbb{R}^M$ denotes the $i$-th sample ($i=1,2,\cdots,N$). Let $\mathbf{S} \in {\mathbb{R}^{N \times N}}$ be the similarity matrix and $\mathbf{S}_{ij}$ corresponds to the similarity between $x_i$ and $x_j$. Besides, consider that $\Phi ( \cdot ):x \in {\mathbb{R}^M} \mapsto \mathcal{H}$ be a feature mapping from $x$ onto a reproducing kernel Hilbert space $\mathcal{H}$ and there exists $\mathcal{K}(x,y) =  < \Phi (x),\Phi (y) >$. Classic clustering methods
usually pre-compute the similarity matrix based on the Euclid distance between pairwise samples. However, we formulate it as a data-based learning problem:
\begin{equation}\label{similarity_learning}
\mathop {\min }\limits_\mathbf{S} \;\mathcal{O}(\mathbf{S}) = \frac{1}{2}\sum\nolimits_{ij} {{\mathbf{S}_{ij}}||{\Phi} ({x_i}) - {\Phi} ({x_j})||_2^2}
\end{equation}
\begin{displaymath}
s.t.\;\mathbf{S} \ge 0,\;\mathbf{S} = {\mathbf{S}^T},\;\mathbf{S}\mathbf{1} = \mathbf{1},
\end{displaymath}
where $\mathbf{S}$ is constrained to become a doubly stochastic matrix for better clustering \cite{Ron-NIPS-DSN-SpectralClustering-2007,Fei-ICDM-Bi-Stochastic-2010}. Note that we utilize the distance between the kernel mappings instead of the original vectors in Eq.~(\ref{similarity_learning}) because $||\cdot||_2$ measures the Euclid space and kernel mappings in the Hilbert space would match such characteristic. Based on the similarity matrix $\mathbf{S}$, the classic spectral embedding methods \cite{Mikhail-Eigenvalues-NIPS-2002,Andrew-NIPS-SpectralClustering-2002,Ron-NIPS-DSN-SpectralClustering-2007,Feiping-AAAI-CLR-2016,Feiping-AAAI-ULGE-2017} could be formulated as below:
\begin{equation}\label{spectral-embedding}
\mathop {\min }\limits_\mathbf{P} \; \frac{1}{2}\sum\nolimits_{ij} {{\mathbf{S}_{ij}}||{\mathbf{P}_i} - {\mathbf{P}_j}||_2^2},\;s.t.\;\mathbf{P}{\mathbf{P}^T} = {\mathbf{I}_{E}},
\end{equation}
where $\mathbf{P}$ is the spectral embedding for the data points and $E$ is the dimension of the embedding vectors. Thereafter, K-Means is adopted to cluster the data samples as a popular post-processing technique. However, we further put the non-negativity on $\mathbf{P}$, set $E$ to $C$ (the number of clusters), and finally arrive at:
\begin{equation}\label{non-negative-spectral-embedding}
\mathop {\min }\limits_\mathbf{P} \;\mathcal{O}(\mathbf{P}) = \frac{1}{2}\sum\nolimits_{ij} {{\mathbf{S}_{ij}}||{\mathbf{P}_i} - {\mathbf{P}_j}||_2^2},\;\\
\end{equation}
\begin{displaymath}
s.t.\;\mathbf{P} \ge 0, \; \mathbf{P}{\mathbf{P}^T} = {\mathbf{I}_{C}}.
\end{displaymath}
Noticeably, each column of $\mathbf{P}$ will be one-hot vector; in other words, $\mathbf{P}$ could be treated as an indicator matrix for clustering. Besides, from the following Theorem~\ref{C_connected_clusters} and Theorem~\ref{ky_fans_theorem}, we can conclude that the optimization~(\ref{non-negative-spectral-embedding}) actually captures the intrinsic structures, i.e., $C$-collected clusters.

\begin{theorem}[$C$-connected clusters \cite{Fan-Mathematical-SpectralGraphTheory-1997}]\label{C_connected_clusters}
``The multiplicity $C$ of the eigenvalue $0$ of the Laplacian matrix $\mathbf{L}_\mathbf{S}$ is equal to the number of connected clusters/components in the graph associated with $\mathbf{S}$.'', which implies:
\begin{equation}
rank({\mathbf{L}_\mathbf{S}}) = N - C \Leftrightarrow \sum\limits_{i = 1}^C {{\lambda _i}}  = 0,
\end{equation}
where ${\mathbf{L}_\mathbf{S}} = diag(\mathbf{S}\mathbf{1})\footnote{$diag( \cdot )$ is a diagonal matrix spanned by the vector parameter. For instance, $diag(\mathbf{1})$ is actually an Identity matrix $\mathbf{I}$. However, if the parameter is a square matrix, then it will returns a column vector with the diagnoal elements, e.g., $diag(\mathbf{I}) = \mathbf{1}$.} - \mathbf{S}$, and $\{\lambda_1,\lambda_2,\cdots,\lambda_N\}$ are the eigenvalues of $\mathbf{L}_\mathbf{S}$ in an ascending order.
\end{theorem}

\begin{theorem}[Ky Fan's Theorem \cite{KyFan-National-Eigenvalues-1949}]\label{ky_fans_theorem}
Given a matrix $\mathbf{P} \in \mathbb{R}^{C \times N}$, the following optimization problem:
\begin{equation}
\mathop {\min }\limits_\mathbf{P} \;\frac{1}{2}\sum\nolimits_{ij} {{\mathbf{S}_{ij}}||{\mathbf{P}_i} - {\mathbf{P}_j}||_2^2}  = tr(\mathbf{P}\cdot{\mathbf{L}_\mathbf{S}}\cdot{\mathbf{P}^T})
\end{equation}
\begin{displaymath}
s.t.\;\mathbf{P}{\mathbf{P}^T} = {\mathbf{I}_C}
\end{displaymath}
is equivalent to $\sum\limits_{i = 1}^C {{\lambda _i}}  \to 0$.
\end{theorem}

Generally, learning with multi-stages, e.g., Spectral Clustering, would usually lead to sub-optimal solutions for clustering. Therefore, we build a marriage between (\ref{similarity_learning}) and (\ref{non-negative-spectral-embedding}) into a joint learning approach:
\begin{equation}\label{overall_objective_function}
\mathop {\min }\limits_{\mathbf{S},\mathbf{P}} \;\mathcal{O}(\mathbf{S},\mathbf{P}),\;s.t.\;\mathbf{S} \in {\mathbf{S}_{dsm}},\;\mathbf{P} \in {\mathbf{P}_{nlr}},
\end{equation}
with
\begin{equation}\label{loss_function}
\mathcal{O}(\mathbf{S},\mathbf{P}) = {\mathcal{O}(\mathbf{S}}) + \alpha ||\mathbf{S}||_F^2 + \beta \mathcal{O}(\mathbf{P}),
\end{equation}
\begin{equation}\label{doubly_stochastic_matrix}
{\mathbf{S}_{dsm}} = \{ \mathbf{S} \ge 0|\mathbf{S} = {\mathbf{S}^T},\mathbf{S}\mathbf{1} = \mathbf{1}\} ,
\end{equation}
\begin{equation}\label{non-negative_low_rank}
{\mathbf{P}_{nlr}} = \{ \mathbf{P} \ge 0|\mathbf{P} \cdot {\mathbf{P}^T} = {\mathbf{I}_C}\}.
\end{equation}
Obviously, $\mathbf{S}_{dsm}$ is a set of doubly stochastic matrices, $\mathbf{P}_{nlr}$ is a set of non-negative low-rank matrices, and $\alpha$ and $\beta$ are two positive hyper-parameters. The philosophy of the optimization (\ref{overall_objective_function}) is an end-to-end single-stage learning for clustering based on non-negative spectral embedding. Therefore, we call our method ``Regularized Non-negative Spectral Embedding (RNSE)'' for Clustering.
\section{Optimization Methods}
Regarding the objective function $\mathcal{O}(\mathbf{S},\mathbf{P})$, there are two coupled variables to be learned which indicates that it's a non-convex optimization problem. Thus, we adopt the classic strategies to address such optimization problem with alternative iterations \cite{Daniel-NMF-learning-objects-Nature-1999,Daniel-NMF-Algorithm-NIPS-2001,Deng-GNMF-TPAMI-2011}, i.e., updating $\mathbf{P}$ while keeping $\mathbf{S}$ fixed and vice versa, until a local minima is achieved. The learning process is narrated in Algorithm~\ref{algo:Framework}. Subsequently, we depict the detailed ideas for solving the two subproblems. 

\begin{algorithm}[t]
\caption{The learning process: RNSE}\label{algo:Framework}
\KwIn{Dataset $\{ {x_i}\} _{i = 1}^N$; hyper-parameters $\alpha$, $\beta$.}
\KwOut{Similarity matrix $\mathbf{S}$ and indicator matrix $\mathbf{P}$.}
\Begin
{
    Randomly initialize indicator matrix $\mathbf{P}$ and similarity matrix $\mathbf{S}$\label{algorithm_line_initialization}\;
    \Repeat{convergence}{
      Optimize problem~(\ref{overall_objective_function}) w.r.t. $\mathbf{S}$ while keeping $\mathbf{P}$ fixed, i.e., Algorithm~\ref{algo:updatingS}\label{algorithm_line_updateS}\;
      Optimize problem~(\ref{overall_objective_function}) w.r.t. $\mathbf{P}$ while keeping $\mathbf{S}$ fixed, i.e., Algorithm~\ref{algo:updatingP}\label{algorithm_line_updateP}\;
    }
}
\end{algorithm}

\subsection{Optimizing $\mathbf{S}$ while keeping $\mathbf{P}$ fixed}
When clustering indicator matrix $\mathbf{P}$ is fixed, the subproblem for optimizing similarity matrix $\mathbf{S}$ can be written as:

\begin{equation}\label{subproblem_update_S}
\mathop {\min }\limits_\mathbf{S} \mathcal{O}(\mathbf{S},\mathbf{P})\;\;s.t.\;\mathbf{S} \in {\mathbf{S}_{dsm}}.
\end{equation}
Set $\mathbf{A} = [\Phi ({x_1}),\Phi ({x_2}), \cdots ,\Phi ({x_N})]$, then we can transform (\ref{subproblem_update_S}) into:
\begin{equation}
\mathop {\min }\limits_\mathbf{S} \;\alpha ||\mathbf{S}||_F^2 + tr(\mathbf{A} \cdot {\mathbf{L}_\mathbf{S}} \cdot {\mathbf{A}^T})+ \beta \cdot  tr(\mathbf{P} \cdot {\mathbf{L}_\mathbf{S}} \cdot {\mathbf{P}^T})
\end{equation}
\begin{displaymath}
s.t.\;\mathbf{S} \in {\mathbf{S}_{dsm}},
\end{displaymath}
which is further equivalent to the simplified formalization:
\begin{equation}\label{subproblem_update_S_final}
\mathop {\min }\limits_\mathbf{S} ||\mathbf{S} - \mathbf{T}||_F^2,\;\;s.t.\;\mathbf{S} \in {\mathbf{S}_{dsm}},
\end{equation}
with $\mathbf{T} = \frac{1}{{2\alpha }}\left[ {({\mathbf{A}^T}\mathbf{A} + \beta{\mathbf{P}^T}\mathbf{P}) - diag({\mathbf{A}^T}\mathbf{A} + \beta{\mathbf{P}^T}\mathbf{P}) \cdot {\mathbf{1}^T}} \right]$.
Essentially, this is an optimization problem to find a doubly stochastic matrix $\mathbf{S}$ nearest to the given matrix $\mathbf{T}$, which could be converted into the ``metric projection optimizations'' and solved with alternating projection methods. 

\begin{definition}[Metric Projection]\label{metric_projection}
Given a set $\mathcal{S} \subseteq \mathcal{H}$ and a point $x \in \mathcal{H}$, the metric projection (if exists) of $x$ onto $\mathcal{S}$ is a point $p \in \mathcal{S}$ such that:
\begin{equation}
||p - x|| = d(x,\mathcal{S}): = \mathop {\inf }\limits_{s \in \mathcal{S}} \;||x - s||.
\end{equation}
Additionally, if for any $x \in \mathcal{H}$, there exists such a unique $p$, then the metric projection onto $\mathcal{S}$ is rewritten as the following operator:
\begin{equation}
{P_\mathcal{S}}: \mathcal{H} \longmapsto \mathcal{S},\;i.e.,\;{P_\mathcal{S}}(x) = p.
\end{equation}
\end{definition}

\begin{theorem}[Projection Theorem]\label{projection_theorem}
Set $\mathcal{C} \subseteq \mathcal{H}$ be a closed convex set. For any $x \in \mathcal{H}$, there exists a unique $p \in \mathcal{C}$ such that $||x - p|| \le ||x - c||$ for all $c \in \mathcal{C}$, which is formally denoted as: 
\begin{equation}
P_\mathcal{C}(x)=p.
\end{equation}
\end{theorem}
\begin{proof}
See Ref.~\cite{Luenberger-Wiley-1969}. 
\end{proof}

\begin{theorem}[Dykstra's Method]\label{Dykstra_Method}
Let ${\mathcal{C}_1}$, ${\mathcal{C}_2}$, $\cdots$, ${\mathcal{C}_r} \subseteq \mathcal{H}$ be closed convex sets and $\mathcal{C}:=\bigcap\limits_{i = 1}^r {{\mathcal{C}_i}}$. If $\mathcal{C} \ne \emptyset$, then given $x \in \mathcal{H}$ iterated by:
\begin{equation}
\left\{ \begin{array}{l}
 x_n^i: = {P_{{\mathcal{C}_i}}}(x_n^{i - 1} - I_{n - 1}^i) \\
 I_n^i: = x_n^i - (x_n^{i - 1} - I_{n - 1}^i) \\
 x_n^0: = x_{n - 1}^r \\
 \end{array} \right.,
\end{equation}
with initial values $x_1^0: = {x},\;I_0^i: = 0$, there holds:
\begin{equation}
\mathop {\lim }\limits_{n \to \infty } {x_n} = {P_\mathcal{C}}(x).
\end{equation}
\end{theorem}
\begin{proof}
See Refs.~\cite{James-statistics-1986}, \cite{Marcos-Mathematics-2011} and \cite{Matthew-Mathematical-AlternatingProjection-2012}. 
\end{proof}

\begin{theorem}\label{closed_convex_sets}
Let $\mathcal{H}=\mathbb{R}^{N \times N}$ and $\mathcal{S} = \{ \mathbf{M} \in \mathcal{H}|\mathbf{M} = {\mathbf{M}^T},\;\mathbf{M}\mathbf{1} = \mathbf{1},\;\mathbf{M} \ge 0\}$, then $\mathcal{S}$ is a closed convex set. 
\end{theorem}
\begin{proof}
This can be easily verified. 
\end{proof}

\begin{algorithm}[t]
\caption{Optimizing $\mathbf{S}$ while keeping $\mathbf{P}$ fixed}\label{algo:updatingS}
\KwIn{Dataset $\{ {x_i}\} _{i = 1}^N$, indicator matrix $\mathbf{P}$, hyper-parameters $\alpha$, $\beta$.}
\KwOut{Similarity matrix $\mathbf{S}$.}
\Begin
{
    Initialize $x_1^0: = {\mathbf{T}}$, $I_0^1: = 0$, $I_0^2: = 0$, $n=1$\;
    \Repeat{convergence}{
      Calculate $x_n^1: = {P_{{C_1}}}(x_n^{0})$\ and $x_n^2: = {P_{{C_2}}}(x_n^{1})$\;
      Calculate $I_n^1: = x_n^1 - (x_n^{0} - I_{n - 1}^1)$ and $I_n^2: = x_n^2 - (x_n^{1} - I_{n - 1}^2)$\;
      Iterative index $n=n+1$\;
      $x_n^0: =x_{n-1}^2$\;
    }
}
\end{algorithm}

Now, let's go back to subproblem~(\ref{subproblem_update_S_final}). Firstly, Theorem~\ref{projection_theorem} and Theorem~\ref{closed_convex_sets} together tell us that there must be one global and unique $\mathbf{S}$ for the optimization problem~(\ref{subproblem_update_S_final}).
Then, inspired by the Theorem~\ref{Dykstra_Method}, we could turn~(\ref{subproblem_update_S_final}) into:
\begin{equation}\label{sub_optimization_PC1}
x_n^1: = {P_{{\mathcal{C}_1}}}(x_n^0),
\end{equation}
and
\begin{equation}\label{sub_optimization_PC2}
x_n^2: = {P_{{\mathcal{C}_2}}}(x_n^1),
\end{equation}
where ${\mathcal{C}_1} = \{\mathbf{S}|\mathbf{S}= {\mathbf{S}^T},\;\mathbf{S}\mathbf{1} = \mathbf{1}\}$, ${\mathcal{C}_2} = \{\mathbf{S}|\mathbf{S} \ge 0\}$, $x_1^0: = {\mathbf{T}}$, $x_n^0: = x_{n-1}^2$, and $n=1,2, \cdots$ denotes the iterative index. The solutions to the two sub-problems \eqref{sub_optimization_PC1} and \eqref{sub_optimization_PC2} could be achieved by Theorem~\ref{P_C_1_Projection} and Theorem~\ref{P_C_2_Projection}, respectively.

\begin{theorem}\label{P_C_1_Projection}
Given any point $\mathbf{T}_v$, the global optimal solution to $P_{\mathcal{C}_1}(\mathbf{T}_v)$ is
\begin{equation}
\mathbf{K} + \frac{{N + {\mathbf{1}^T}\mathbf{K}\mathbf{1}}}{{{N^2}}}{\mathbf{1}\mathbf{1}^T} - \frac{1}{N}\mathbf{K}{\mathbf{1}\mathbf{1}^T} - \frac{1}{N}{\mathbf{1}\mathbf{1}^T}\mathbf{K},
\end{equation}
where $\mathbf{K} = \frac{{\mathbf{T}_v + {\mathbf{T}_v^T}}}{2}$.
\end{theorem}
\begin{proof}
See Ref.~\cite{Xiaoqian-KDD-2016}. 
\end{proof}

\begin{theorem}\label{P_C_2_Projection}
Given any point $\mathbf{T}_v$, the global optimal solution to $P_{\mathcal{C}_2}(\mathbf{T}_v)$ is
\begin{equation}
[\mathbf{T}_v]_{+},
\end{equation}
where ${[\cdot]_ + }$ is an element-wise non-negative operation.
\end{theorem}
\begin{proof}
This can be easily verified. 
\end{proof}

Up to present, we could achieve the optimal solution for problem~(\ref{subproblem_update_S_final}) through Dykstra's method, which is concluded in Algorithm~\ref{algo:updatingS}.

\subsection{Optimizing $\mathbf{P}$ while keeping $\mathbf{S}$ fixed}

It is straightforward to see that, when similarity matrix $\mathbf{S}$ is fixed, the optimization problem~(\ref{overall_objective_function}) could be reduced as:
\begin{equation}\label{subproblem_update_P}
\mathop {\min }\limits_{\mathbf{P}} \;\mathcal{O}(\mathbf{P}),\;\;\;s.t.\;\mathbf{P} \in {\mathbf{P}_{nlr}}.
\end{equation}

Regarding that $\mathbf{P}$ is constrained to be both orthogonal and non-negative in (\ref{subproblem_update_P}), it seems quite challenging to deal with such problem.
Subtly, since the similarity matrix $\mathbf{S}$ is a doubly stochastic matrix, then we could draw the following Theorem~\ref{P_transformer}.

\begin{theorem}[$\mathbf{P}$-Transformer]\label{P_transformer}
Given $\mathbf{S}$ is a doubly stochastic matrix, the subproblem~(\ref{subproblem_update_P}) could be converted to the following optimization problem:
\begin{equation}\label{equivalent_P}
\mathop {\min }\limits_{\mathbf{P}} \;||\mathbf{S} - {\mathbf{P}^T}\mathbf{P}||_F^2,\;\;s.t.\;\mathbf{P} \in {\mathbf{P}_{nlr}}.
\end{equation}
\end{theorem}
\begin{proof}
Since $\mathbf{S} \in {\mathbf{S}_{dsm}}$, then there is $\mathbf{L}_\mathbf{S} = \mathbf{I}_C -\mathbf{S}$, which is followed by:
\begin{equation}\nonumber
\begin{aligned}
 tr(\mathbf{P} \cdot {\mathbf{L}_\mathbf{S}} \cdot {\mathbf{P}^T}) = C - tr(\mathbf{P} \cdot \mathbf{S} \cdot {\mathbf{P}^T}).
 \end{aligned}
\end{equation}

Then the objective function in~(\ref{subproblem_update_P}) could be further transformed into:
\begin{equation}
\begin{aligned}
 &  \mathop {\min }\limits_\mathbf{P}\;  tr(\mathbf{P} \cdot {\mathbf{L}_\mathbf{S}} \cdot {\mathbf{P}^T})   \\
\propto  &\mathop {\max }\limits_\mathbf{P}\; tr(\mathbf{P} \cdot \mathbf{S} \cdot {\mathbf{P}^T})    \\
\propto  &\mathop {\min }\limits_\mathbf{P}\; \{ tr(\mathbf{S} \cdot {\mathbf{S}^T})- 2tr(\mathbf{S} \cdot {\mathbf{P}^T}\mathbf{P}) + tr({\mathbf{P}^T}\mathbf{P}{\mathbf{P}^T}\mathbf{P})\}  \\
\propto  &\mathop {\min }\limits_\mathbf{P}\; ||\mathbf{S} - {\mathbf{P}^T}\mathbf{P}||_F^2.
 \end{aligned}
\end{equation}
Consequently, the two optimization problems~(\ref{subproblem_update_P}) and~(\ref{equivalent_P}) are equivalent. 
\end{proof}

In light of the optimization problem~(\ref{equivalent_P}), its augmented Lagrange function is displayed as:
\begin{equation}
\mathcal{L}(\mathbf{P},\Lambda ) = ||\mathbf{S} - {\mathbf{P}^T}\mathbf{P}||_F^2 + tr\{ \Lambda (\mathbf{P}{\mathbf{P}^T} - {\mathbf{I}_C})\},
\end{equation}
where $\Lambda$ denotes the Lagrange multiplier for the constraint $\mathbf{P}{\mathbf{P}^T}=\mathbf{I}_C$. With respect to the non-negative constraints, they could be ignored in the augmented Lagrange function when the multiplicative update philosophy \cite{Daniel-NMF-Algorithm-NIPS-2001,Ziyu-MV-NMF-TKDE-2015,Yu-Survry-NMF-TKDE-2013} is adopted, because the non-negativity for variable $\mathbf{P}$ is naturally maintained during iterative updatings. More specifically, the derivative of $\mathcal{L}(\mathbf{P},\Lambda )$ w.r.t. $\mathbf{P}$ can be formulated into two parts, i.e.,
\begin{equation}\label{derivative_over_P}
\begin{aligned}
\frac{{\partial \mathcal{L}(\mathbf{P},\Lambda )}}{{\partial \mathbf{P}}}&= 4\mathbf{P}{\mathbf{P}^T}\mathbf{P} - 2\mathbf{P}(\mathbf{S} + {\mathbf{S}^T}) + (\Lambda  + {\Lambda ^T})\mathbf{P}\\
&\buildrel \Delta \over = {\left( \mathcal{A} \right)_ + } - {\left( \mathcal{B} \right)_ + } ,
\end{aligned}
\end{equation}
where ${\left( \mathcal{A} \right)_ + }$ and ${\left( \mathcal{B} \right)_ + }$ both represent element-wise non-negative matrices.
Then the iterative formula for updating $\mathbf{P}$ could be written as:
\begin{equation}\label{iterative_P}
\mathbf{P} \leftarrow \mathbf{P} \otimes {\left( \mathcal{B} \right)_ + } \oslash {\left( \mathcal{A} \right)_ + },
\end{equation}
where $\otimes$ and $\oslash$ corresponds to the element-wise multiplication and element-wise division respectively \cite{Jiho-ONMF-IDEAL-2008}. Obviously, if the factors in Equation~(\ref{iterative_P}) are all non-negative, then the result will also hold non-negativity. In order to figure out the detailed formula for Equation~(\ref{iterative_P}), we have to derive $\Lambda$.

\begin{algorithm}[t]
\caption{Optimizing $\mathbf{P}$ while keeping $\mathbf{S}$ fixed}\label{algo:updatingP}
\KwIn{Similarity matrix $\mathbf{S}$ and  hyper-parameters $\lambda$, $\mu$}
\KwOut{Clustering indicator matrix $\mathbf{P}$.}
\Begin
{
    Randomly initialize clustering indicator matrix $\mathbf{P}$\;
    \Repeat{convergence}{
      Update $\mathbf{P}$ according to the strategic multiplicative update rule~(\ref{iterative_P_elementwise_combination})\;
    }
}
\end{algorithm}
Taking the partial derivatives of $\mathcal{L}(\mathbf{P},\Lambda)$ w.r.t. $\mathbf{P}$ and $\Lambda$ respectively, and setting them to zero by the Karush-Kuhn-Tucker conditions, we arrive at:
\begin{equation}\label{KKT_partial_L2P}
\frac{{\partial \mathcal{L}(\mathbf{P},\Lambda )}}{{\partial \mathbf{P}}} = 4\mathbf{P}{\mathbf{P}^T}\mathbf{P} - 2\mathbf{P}(\mathbf{S} + {\mathbf{S}^T}) + (\Lambda  + {\Lambda ^T})\mathbf{P} = 0,
\end{equation}
and
\begin{equation}\label{KKT_partial_L2Lambda}
\frac{{\partial \mathcal{L}(\mathbf{P},\Lambda )}}{{\partial \Lambda }} = \mathbf{P}{\mathbf{P}^T} - {\mathbf{I}_C} = 0.
\end{equation}

Multiplying factor $\mathbf{P}^T$ on both sides of Equation~(\ref{KKT_partial_L2P}) and then taking the Equation~(\ref{KKT_partial_L2Lambda}) into it, there comes:
\begin{equation}\label{KKT_derivation_Lambda}
\Lambda  + {\Lambda ^T} = 2\mathbf{P}(\mathbf{S} + {\mathbf{S}^T}){\mathbf{P}^T} - 4{\mathbf{I}_C}.
\end{equation}

Combining Equation~(\ref{KKT_derivation_Lambda}) with Equation~(\ref{derivative_over_P}), we could easily obtain:
\begin{equation}
\begin{split}
&{\left( \mathcal{A} \right)_ + } = 4\mathbf{P}{\mathbf{P}^T}\mathbf{P} + 2\mathbf{P}(\mathbf{S} + {\mathbf{S}^T}){\mathbf{P}^T}\mathbf{P};\;\;\; \\
&{\left( \mathcal{B} \right)_ + } = 2\mathbf{P}(\mathbf{S} + {\mathbf{S}^T}) + 4\mathbf{P}.\\
\end{split}
\end{equation}

Then the specific formula for updating $\mathbf{P}$ is expressed in the following:
\begin{equation}\label{iterative_P_matrix}
\begin{aligned}
\mathbf{P} \leftarrow \mathbf{P} &\otimes \{ \mathbf{P}(\mathbf{S} + {\mathbf{S}^T}) + 2\mathbf{P}\} \\
 &\oslash \{ 2\mathbf{P}{\mathbf{P}^T}\mathbf{P} + \mathbf{P}(\mathbf{S} + {\mathbf{S}^T}){\mathbf{P}^T}\mathbf{P}\} ,\\
\end{aligned}
\end{equation}
or in an element-wise version:
\begin{equation}\label{iterative_P_elementwise}
{\mathbf{P}_{ij}} \leftarrow {\mathbf{P}_{ij}}\frac{{{{\left[ {\mathbf{P}(\mathbf{S} + {\mathbf{S}^T}) + 2\mathbf{P}} \right]}_{ij}}}}{{{{\left[ {2\mathbf{P}{\mathbf{P}^T}\mathbf{P} + \mathbf{P}(\mathbf{S} + {\mathbf{S}^T}){\mathbf{P}^T}\mathbf{P}} \right]}_{ij}}}},
\end{equation}
where $i = 1,2, \cdots ,C$ and $j = 1,2, \cdots ,N.$
Inspired by Ref.~\cite{Zhaoshui-SymNMF-TNNLS-2011}, using Equation~(\ref{iterative_P_matrix} or \ref{iterative_P_elementwise}) to update $\mathbf{P}$ directly would yield unstable performance,
then we adopt the following ``strategic multiplicative updates'':

\begin{equation}\label{iterative_P_elementwise_combination}
{\mathbf{P}_{ij}}\leftarrow{\mathbf{P}_{ij}}{\left\{ {(1 - \lambda ) +  \lambda\mathbf{Q}} \right\}^\mu }
\end{equation}
where $\mathbf{Q}=\frac{{{{\left[ {\mathbf{P}(\mathbf{S} + {\mathbf{S}^T}) + 2\mathbf{P}} \right]}_{ij}}}}{{{{\left[ {2\mathbf{P}{\mathbf{P}^T}\mathbf{P} + \mathbf{P}(\mathbf{S} + {\mathbf{S}^T}){\mathbf{P}^T}\mathbf{P}} \right]}_{ij}}}}$, $\lambda$ and $\mu$ are set to $0.5$ and $0.9$ respectively in our experimental parts. Here we call the iterative rule (\ref{iterative_P_elementwise_combination}) ``Strategic Multiplicative Update''. Thus the procedure for optimizing $\mathbf{P}$ is summarized in Algorithm~\ref{algo:updatingP}.

\section{Experiments}
This section would empirically evaluate the effectiveness of RNSE for clustering on both \textit{synthetic} and \textit{real-world} datasets.

\subsection{Experiments on Synthetic Datasets}
The first synthetic dataset we constructed is a $1000\times1000$ matrix with four $250\times250$ block matrices diagonally arranged (Fig.~\ref{block-diagonal}). The data in each block denotes the affinity of any two points within one cluster while the data outside all blocks denotes noise. The affinity data within each block is randomly generated in the range from $0$ to $1$, while the noise is randomly generated in the range from $0$ to $c$, which is set as $0.2$, $0.4$ and $0.8$ respectively during the experiments. 

Fig.~\ref{block-diagonal} exhibits the original graphs and their corresponding clustering results under different noise (e.g. $c$) settings. We can see that RNSE overall presents good performances w.r.t. clustering task. Specifically, RNSE successfully learns a structured doubly stochastic matrix with explicit block structures, which divided the data samples into four clear clusters. As the noise increases, the block structure in the original graph blurs, but RNSE is still able to detect the intrinsic structures of the data, which indicates the robustness of the RNSE method for potential practical applications.

\begin{figure*}[htb]
\centering
\subfigure[Original Graph, noise = 0.2]{
\begin{minipage}[t]{0.33\linewidth}
\centering
\includegraphics[width=2.5in]{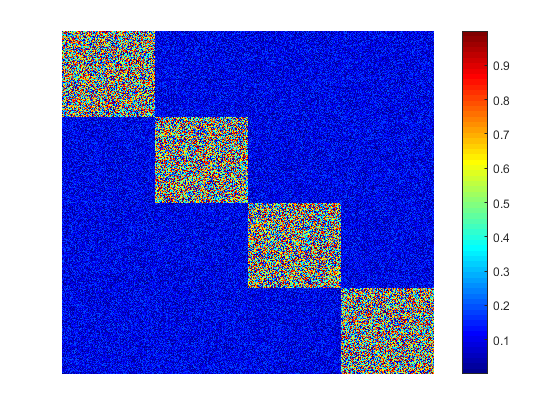}
\end{minipage}%
}%
\subfigure[Original Graph, noise = 0.4]{
\begin{minipage}[t]{0.33\linewidth}
\centering
\includegraphics[width=2.5in]{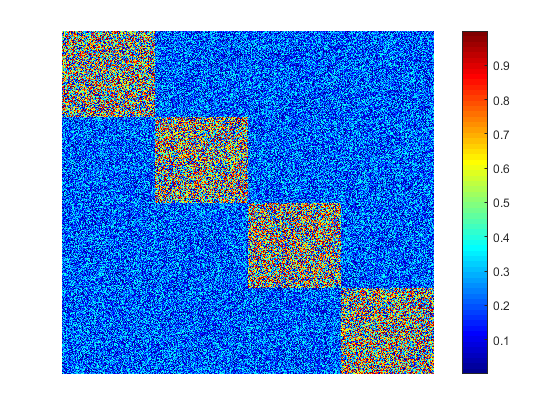}
\end{minipage}%
}%
\subfigure[Original Graph, noise = 0.8]{
\begin{minipage}[t]{0.33\linewidth}
\centering
\includegraphics[width=2.5in]{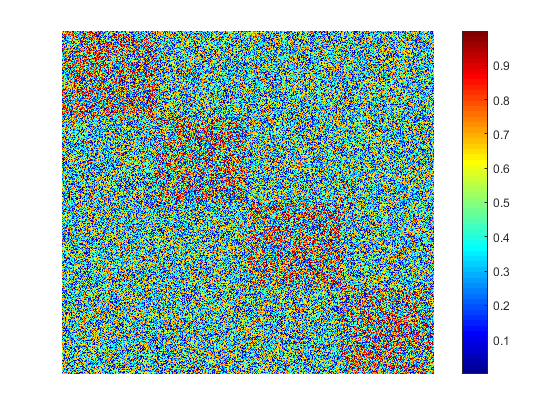}
\end{minipage}
}%

\centering
\subfigure[RNSE Result, noise = 0.2]{
\begin{minipage}[t]{0.33\linewidth}
\centering
\includegraphics[width=2.5in]{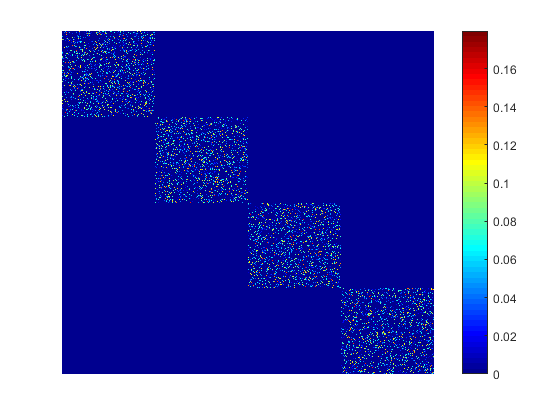}
\end{minipage}%
}%
\subfigure[RNSE Result, noise = 0.4]{
\begin{minipage}[t]{0.33\linewidth}
\centering
\includegraphics[width=2.5in]{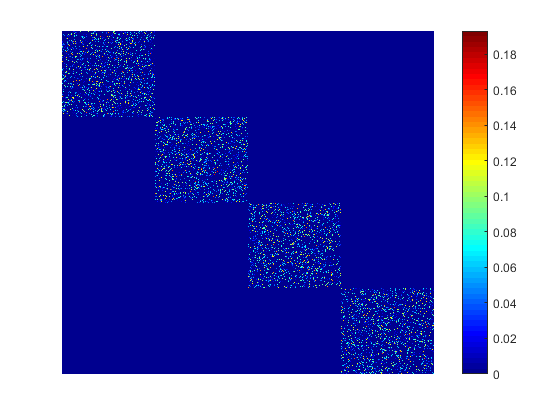}
\end{minipage}%
}%
\subfigure[RNSE Result, noise = 0.8]{
\begin{minipage}[t]{0.33\linewidth}
\centering
\includegraphics[width=2.5in]{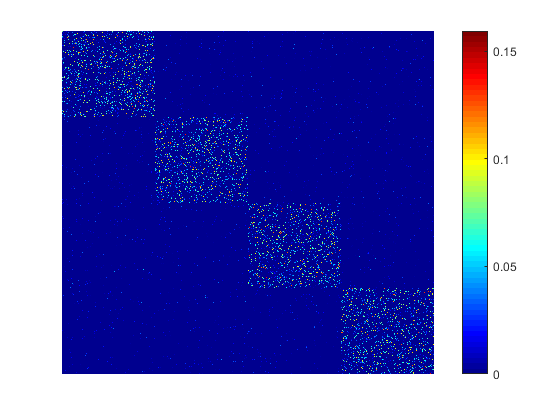}

\end{minipage}
}%

\centering
\caption{Clustering results on the \textit{block-diagonal} synthetic dataset.}
\label{block-diagonal}
\end{figure*}

The second synthetic dataset is a randomly generated two-moon dataset. There are two clusters with each being a volume of $100$ samples distributed in the moon shape (Fig.~\ref{two_moon}). Here, we tested K-Means\cite{Stuart-IEEE-K-means-1982}, Ncuts\cite{Jianbo-PAMI-N-cut-2000,Andrew-NIPS-SpectralClustering-2002} and RNSE on such dataset. Note that in this figure, the color of the two clusters are set to be red and blue, respectively; and the green lines denote the affinity of any two points. Obviously on the whole, Fig.~\ref{two_moon} tells us the RNSE's effectiveness. 

More specifically, some analysis could be drawn as follows. First, there are some points split into the wrong cluster w.r.t K-Means (Fig.~\ref{two_moon}(b)). It's easy to understand that K-Means mainly deals with ball-like distributed data points, which is obviously not fit for the manifold (e.g. two-moon) data points. Second, From Fig.~\ref{two_moon}(c), NCuts could well divide the data points into two separate clusters, but the green lines (affinities) between samples are mixed across different clusters, which indicates that the classic spectral clustering methods tends to mis-recognize neighbors. Third, RNSE (Fig.~\ref{two_moon}(d)), extended from the spectral family, could tell the differences  between both classes and neighbors. This implies that RNSE potentially holds stronger abilities than the classic spectral methods (e.g. NCuts) in handling complex datasets.

\begin{figure*}[htb]
\centering
\subfigure[Original Points]{
\begin{minipage}[t]{0.24\linewidth}
\centering
\includegraphics[width=1.9in]{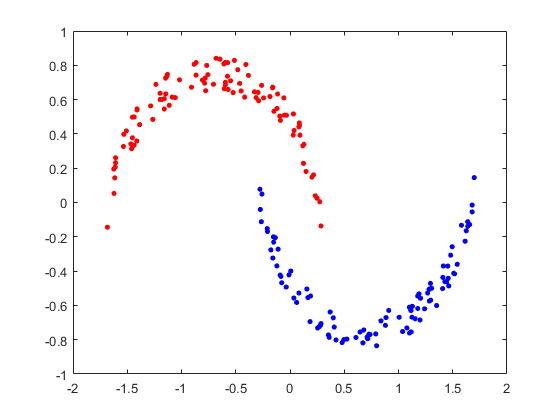}
\end{minipage}%
}%
\subfigure[K-Means Result]{
\begin{minipage}[t]{0.24\linewidth}
\centering
\includegraphics[width=1.9in]{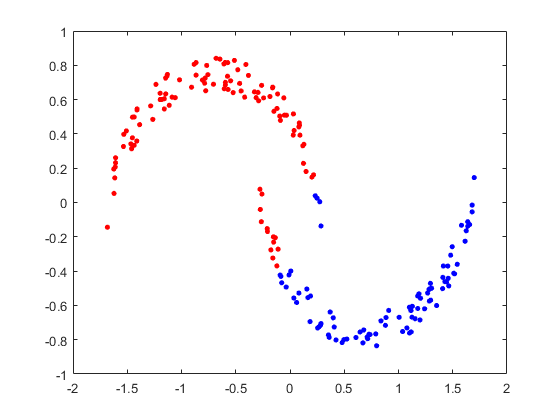}
\end{minipage}%
}%
\subfigure[NCuts Result]{
\begin{minipage}[t]{0.24\linewidth}
\centering
\includegraphics[width=1.9in]{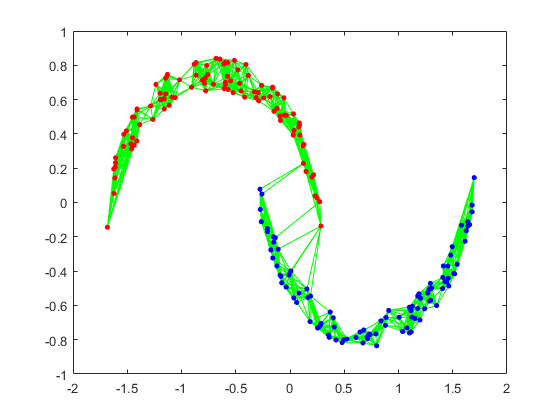}
\end{minipage}
}%
\centering
\subfigure[RNSE Result]{
\begin{minipage}[t]{0.24\linewidth}
\centering
\includegraphics[width=1.9in]{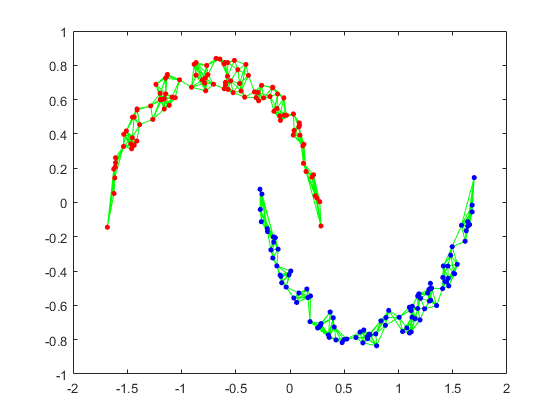}
\end{minipage}%
}%

\centering
\caption{Clustering results on the \textit{two-moon} synthetic dataset.}
\label{two_moon}
\end{figure*}

\begin{figure*}[htb]
\centering
\subfigure[diabetes]{
\begin{minipage}[t]{0.24\linewidth}
\centering
\includegraphics[width=1.7in]{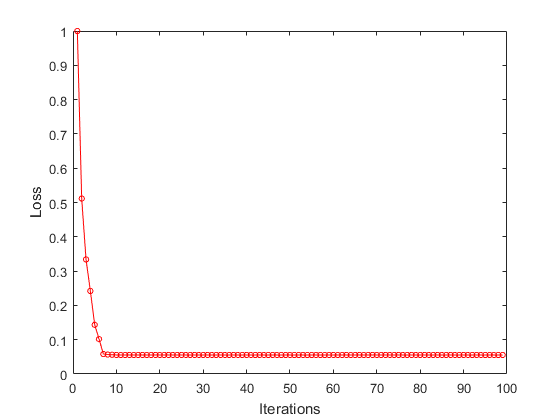}
\end{minipage}%
}%
\centering
\subfigure[arcene]{
\begin{minipage}[t]{0.24\linewidth}
\centering
\includegraphics[width=1.7in]{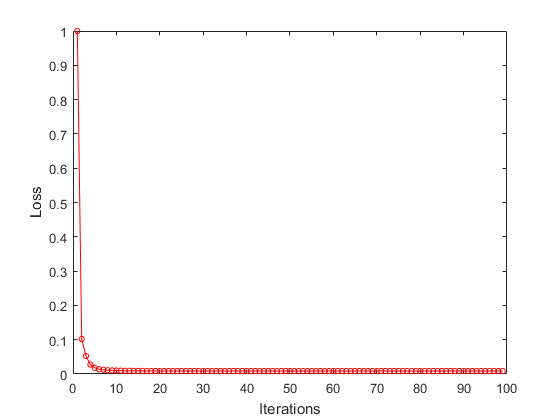}
\end{minipage}%
}%
\subfigure[mnist]{
\begin{minipage}[t]{0.24\linewidth}
\centering
\includegraphics[width=1.7in]{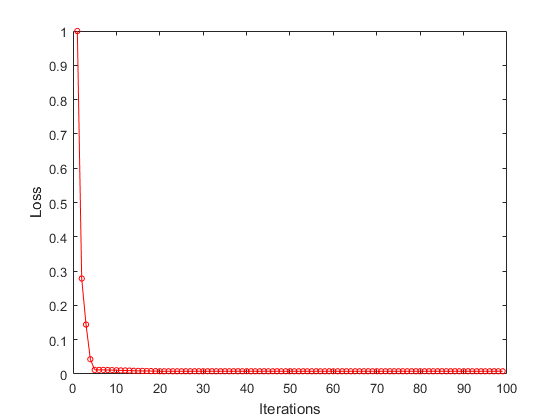}
\end{minipage}%
}%
\subfigure[alpha\_digit]{
\begin{minipage}[t]{0.24\linewidth}
\centering
\includegraphics[width=1.7in]{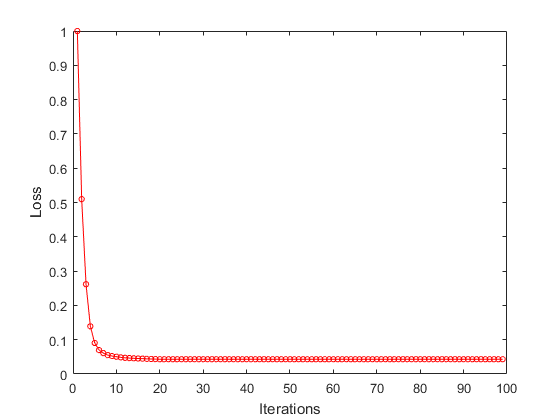}
\end{minipage}
}%

\centering
\subfigure[yeast\_uni]{
\begin{minipage}[t]{0.24\linewidth}
\centering
\includegraphics[width=1.7in]{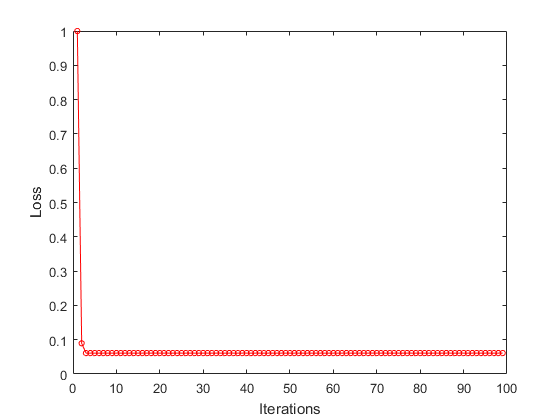}
\end{minipage}%
}%
\subfigure[PCMAC]{
\begin{minipage}[t]{0.24\linewidth}
\centering
\includegraphics[width=1.7in]{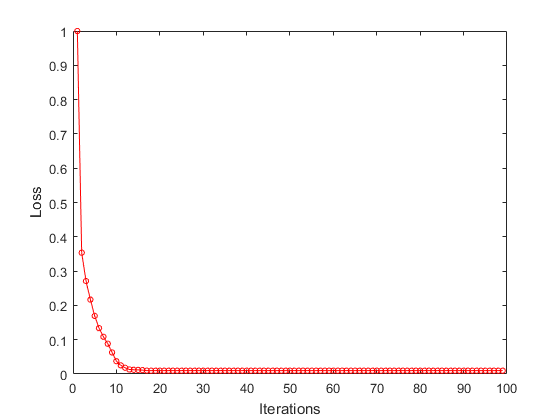}
\end{minipage}%
}%
\subfigure[waveform-21]{
\begin{minipage}[t]{0.24\linewidth}
\centering
\includegraphics[width=1.7in]{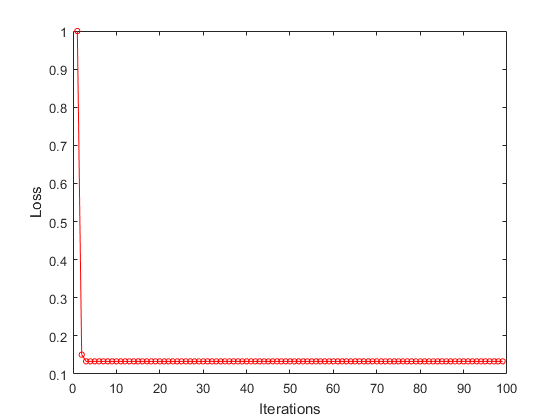}
\end{minipage}
}%
\subfigure[gisette]{
\begin{minipage}[t]{0.24\linewidth}
\centering
\includegraphics[width=1.7in]{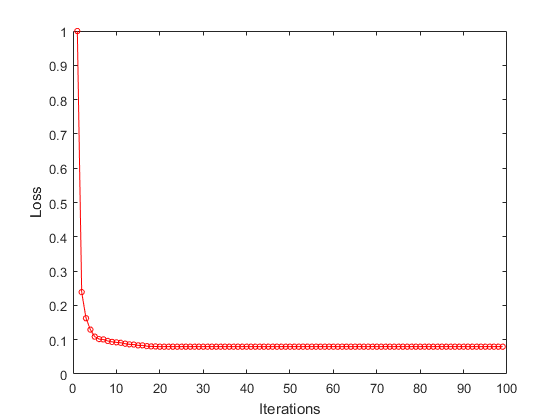}
\end{minipage}
}%
\caption{Convergence curves of RNSE on different datasets: the $y$-axis is the normalized objective function value (loss) and $x$-axis is the iteration number.}  
\label{convergence-curves}   
\end{figure*}

\subsection{Experiments on Real-world Datasets}
\textbf{Datasets.} Eight real-world datasets are selected in the clustering experiments. More specifically, the ``diabetes'', ``arcene'',``yeast\_uni'', ``waveform\-21'' and ``gisette'' are 5 publicly available collections from website\footnote{\url{https://archive.ics.uci.edu/ml/datasets.php}}; the ``PCMAC'' datasete is available from website\footnote{\url{http://featureselection.asu.edu/datasets.php}}; while the ``mnist'' and ``alpha-digit'' datasets are collected from Sam Roweis' page\footnote{\url{https://cs.nyu.edu/home/index.html}}. Note that we just select the top-100 samples of each digit (``0''$\sim$``9'') in ``mnist'' dataset for our experiments. The detailed statistics are summarized in Table~\ref{datasets-statiscits}.

\textbf{Competing Methods.} To demonstrate the effectiveness of RNSE, we compare it with several popular clustering algorithms, i.e.,
(1) Canonical K-Means (K-Means)\cite{Stuart-IEEE-K-means-1982}\footnote{\url{http://scikit-learn.org/stable/modules/generated/sklearn.cluster.KMeans}},
(2) Principal Component Analysis (PCA)\cite{SVANTE-Chemometrics-PCA-1987}\footnote{\url{http://www.cad.zju.edu.cn/home/dengcai/Data/DimensionReduction.html}},
(3) K-Means clustering in the Low-Rank subspaces (LRR)\cite{Guangcan-PAMI-LRR-2013}\footnote{\url{https://sites.google.com/site/guangcanliu/}},
(4) Non-negative Matrix Factorization (NMF)\cite{Daniel-NMF-learning-objects-Nature-1999,Wei-NMFclustering-SIGIR-2003}\footnote{\url{http://scikit-learn.org/stable/modules/generated/sklearn.decomposition.NMF}},
(5) Normalized Cut (NCuts)\cite{Jianbo-PAMI-N-cut-2000,Andrew-NIPS-SpectralClustering-2002}\footnote{\url{https://scikit-learn.org/stable/modules/generated/sklearn.cluster.SpectralClustering}},
(6) Structured Doubly Stochastic Matrix (SDS)\cite{Xiaoqian-KDD-2016}, (7) Clustering with Adaptive Neighbors (CAN)\cite{Feiping-KDD-CAN-2014} and
(8) Constrained Laplacian Rank Algorithm for Graph-based Clustering (CLR)\cite{Feiping-AAAI-CLR-2016}\footnote{SDS, CAN and CLR are available at \url{http://www.escience.cn/people/fpnie/index.html}}.

\begin{table}[tb]
	\centering
	\caption{Statistics of datasets.}
	\label{datasets-statiscits}
	{
		\begin{tabular}{l|rrr}
			\toprule
			Datasets     		& \#Samples & \#Features & \#Classes  \\
			\midrule
			diabetes			& 768	& 8	    & 2     \\
			arcene				& 900	& 1000	& 2		\\
			mnist		        & 1000	& 784	& 10	\\				
			alpha\_digit        & 1404	& 320	& 36 	\\
			yeast\_uni          & 1484	& 1470	& 10	\\
			PCMAC				& 1943  & 3289  & 2     \\
			waveform-21	    	& 2746	& 21	& 3		\\
			gisette				& 7000  & 5000  & 2		\\
			\bottomrule
		\end{tabular}
	}
\end{table}
\begin{table*}[htb]
\centering
\caption{Different competitors' clustering performances measured by ACC (\%).}
\label{Clustering-results-ACC}
{
\begin{tabular}{l|rrrrrrrr|r}
\toprule
Dataset		&K-Means	&PCA	&LRR	&NMF	&Ncuts	&SDS	&CAN	&CLR	&RNSE\\
\midrule
diabetes		&$51.6\pm0.0$	&$51.6\pm0.0$	&$52.6\pm0.0$	&$52.0\pm0.1$	&$51.6\pm0.0$	&$51.6\pm0.0$	&$48.7\pm0.0$	&$64.1\pm0.0$	& $\mathbf{64.7\pm0.0}$\\
arcene			&$59.0\pm0.0$	&$59.0\pm0.0$	&$59.0\pm0.0$	&------------	&$59.0\pm0.0$	&$59.0\pm0.0$	&$40.0\pm0.0$	&$50.5\pm0.0$	& $\mathbf{62.5\pm0.0}$\\
mnist			&$51.0\pm0.2$	&$49.3\pm0.2$	&$54.7\pm0.0$	&------------	&$56.1\pm0.7$	&$53.2\pm0.7$	&$55.3\pm0.0$	&$56.6\pm0.0$	& $\mathbf{64.0\pm0.3}$\\
alpha\_digit	&$43.1\pm1.9$	&$45.6\pm1.2$	&$37.8\pm0.8$	&$38.9\pm1.0$	&$39.9\pm1.3$	&$46.4\pm1.0$	&$19.8\pm0.0$	&$28.9\pm0.0$	& $\mathbf{48.1\pm0.8}$\\
yeast\_uni		&$33.7\pm3.9$	&$37.4\pm1.2$	&$17.8\pm0.1$	&------------	&$29.4\pm0.1$	&$35.4\pm0.8$	&$30.8\pm0.1$	&$37.4\pm0.0$	& $\mathbf{41.3\pm0.2}$\\
PCMAC			&$55.4\pm0.0$	&$55.3\pm0.0$	&$55.7\pm0.0$	&$56.8\pm0.0$	&$55.5\pm0.0$	&$51.6\pm0.0$	&$50.6\pm0.0$	&$50.3\pm0.0$	& $\mathbf{59.3\pm0.0}$\\
waveform-21		&$50.7\pm0.1$	&$50.7\pm0.0$	&$46.1\pm0.0$	&------------	&$50.9\pm0.0$	&$51.9\pm0.0$	&$47.5\pm0.0$	&$50.7\pm0.0$	& $\mathbf{52.7\pm0.0}$\\
gisette			&$51.5\pm0.1$	&$49.3\pm0.1$	&$50.6\pm0.1$	&------------	&$92.9\pm0.1$	&$93.3\pm0.0$	&$91.4\pm0.1$	&$93.8\pm0.1$	& $\mathbf{94.9\pm0.1}$\\
\bottomrule
\multicolumn{10}{l}{Note: ``---'' denotes that the mixed signed matrices for those datasets are not suitable for the MU algorithms employed by NMF.}
\end{tabular}}

\end{table*}

\begin{table*}[htb]
\centering

\caption{Different competitors' clustering performances measured by Purity (\%).}
\label{Clustering-results-purity}
{
\begin{tabular}{l|rrrrrrrr|r}
\toprule
Dataset		&K-Means	&PCA	&LRR	&NMF	&Ncuts	&SDS	&CAN	&CLR	&RNSE\\
\midrule
diabetes		&$65.1\pm0.0$	&$65.1\pm0.0$	&$65.1\pm0.0$	&$65.1\pm0.0$	&$65.1\pm0.0$	&$65.1\pm0.0$	&$\mathbf{65.7\pm0.0}$	&$65.3\pm0.0$	& $\mathbf{65.7\pm0.0}$\\
arcene			&$59.0\pm0.0$	&$59.0\pm0.0$	&$59.0\pm0.0$	&------------	&$59.0\pm0.0$	&$59.0\pm0.0$	&$40.0\pm0.0$	&$50.5\pm0.0$	& $\mathbf{62.5\pm0.0}$\\
mnist			&$56.1\pm1.7$	&$55.5\pm1.4$	&$59.9\pm0.0$	&------------	&$60.6\pm0.7$	&$59.0\pm6.8$	&$64.1\pm0.0$	&$57.0\pm0.0$	& $\mathbf{68.1\pm0.3}$\\
alpha\_digit	&$46.5\pm1.6$	&$48.9\pm0.9$	&$41.2\pm0.4$	&$41.9\pm1.0$	&$42.8\pm1.3$	&$49.7\pm0.7$	&$19.8\pm0.0$	&$33.2\pm0.0$	& $\mathbf{50.5\pm0.3}$\\
yeast\_uni		&$47.1\pm1.7$	&$38.8\pm1.2$	&$33.6\pm0.2$	&------------	&$31.3\pm0.0$	&$37.3\pm0.4$	&$46.4\pm0.1$	&$46.4\pm0.0$	& $\mathbf{47.0\pm0.3}$\\
PCMAC			&$55.4\pm0.0$	&$55.4\pm0.0$	&$55.7\pm0.0$	&$56.8\pm0.1$	&$55.5\pm0.0$	&$51.6\pm0.0$	&$50.5\pm0.0$	&$50.5\pm0.0$	& $\mathbf{59.3\pm0.0}$\\
waveform-21		&$53.3\pm0.0$	&$53.5\pm0.0$	&$47.8\pm0.0$	&------------	&$51.7\pm0.0$	&$52.0\pm0.0$	&$\mathbf{52.8\pm0.0}$	&$\mathbf{52.8\pm0.0}$	& $\mathbf{52.8\pm0.0}$\\
gisette			&$69.3\pm0.1$	&$67.8\pm0.1$	&$50.6\pm0.1$	&------------	&$67.9\pm0.1$	&$94.3\pm0.0$	&$93.9\pm0.1$	&$94.0\pm0.1$	& $\mathbf{94.9\pm0.1}$\\
\bottomrule
\end{tabular}}

\end{table*}

Among these algorithms, NCuts, SDS, CAN, CLR and RNSE are approaches that consider the graph-based structures. PCA and LRR seek the low-rank principal components for data compression. NMF is a model with non-negative constraints and thus could learn additive parts-based components. Note that K-Means directly splits the original high-dimensional data points into clusters; while our RNSE method could also cluster the datasets directly owing to the end-to-end single-stage learning for indicator matrix. However, other methods (excluding K-Means, CAN, CLR and RNSE) mainly conduct two-stage learning for data clustering, i.e., low-dimensional embedding and K-Means clustering.

\textbf{Evaluation Metrics.} Accuracy (ACC) and Purity are widely accepted and therefore employed here to assess the clustering performance.
The higher the performance scores, the better the clustering results, based on the ground-truth class labels presented in the datasets.
For more details, please refer to Refs.~\cite{Jin-TKDD-RobustManifoldNMF-2013} and \cite{Feiping-Cybernetics-IndependentClustering-2012}.

\textbf{Experimental Settings.} In light of the experimental settings, all the baseline methods adopt the best parameter configurations as suggested in their corresponding papers. 
We use the widely used self-tune Gaussian method ~\cite{zelnik2005self} to construct the affinity matrix with kernel function $\mathcal{K}(x,y) = \exp \{- {(x - y)^2}/\gamma \}$ (the value of $\gamma$ is self-tuned, for not only RNSE but also SDS, CLR, and CAN as suggested in their corresponding papers).

As to our RNSE approach, we  empirically achieve the competitive $\alpha$ and $\beta$ according to the $6 \times 6$ grid searches with ${10^{[-3:1:2]}}$ and ${10^{[-5:1:0]}}$, respectively. 
Generally speaking, the best clustering performances on different datasets correspond to slightly different parameter settings. But around the parameters $\alpha=1$ and $\beta=1$ , all have achieved competitive experimental results.

In addition, the maximum iterations for Algorithm~\ref{algo:updatingS} and Algorithm~\ref{algo:updatingP} are set to $20$ and $20$ respectively, and the convergence precision is configured as $10^{-9}$ for these two sub-algorithms. Based on such settings, the number of outer cycles for Algorithm~\ref{algo:Framework} is set to $20$ for convergence. 

\textbf{Clustering Results.}
In this part, we collect the average clustering results in terms of ACC and Purity for all the algorithms on these datasets and show them in Table~\ref{Clustering-results-ACC} and Table~\ref{Clustering-results-purity}.
Note that for all datasets, $\alpha$ and $\beta$ in RNSE are both set to $1$, we repeat the experiments for $10$ times and average the metric values as the final results.
Broadly speaking, different clustering approaches perform differently on various datasets. From Table~\ref{Clustering-results-ACC}, we can easily figure out that LRR presents high ACC values on ``PCMAC'' dataset while performs quite poorly on ``yeast\_uni'' and ``alpha\_digit'' datasets. This phenomenon also appears in Table~\ref{Clustering-results-purity} w.r.t. Purity. It's probably due to the complex structures of ``yeast\_uni'' (the cellular localization sites of proteins) and ``alpha\_digit'' (containing both letters and numbers), which are not well match for the low-rank assumptions. In terms of NCuts, a spectral-based method, always performs quite well on various datasets, which is reasonable because spectral-based algorithms could capture the local structures by keeping the neighborhood similarities and therefore preserve the nonlinear manifolds on complex datasets. However, it's usually inferior to the best performers, and this is probably owing to the self-defined similarities and multi-stage learning. In light of other competitors (taking SDS for example), they also display the similar patterns, i.e., yielding high values on some datasets (``yeast\_uni'' or ``waveform-21''), but meanwhile showing poor performance on some other datasets (``alpha\_digit'' or ``PCMAC'').
Nevertheless, as one can see clearly from all the experimental results in Table~\ref{Clustering-results-ACC} and Table~\ref{Clustering-results-purity}, our RNSE method consistently achieves the best or at least comparative performances on all the datasets regarding ACC and Purity.
This confirms that by designing an end-to-end single-stage learning paradigm with structured constraints, RNSE could better capture the hidden complex structures and thus learn a well-performed indicator matrix for clustering.

\textbf{Convergence \& Complexity Analysis.}
The updating rules in Algorithm~\ref{algo:Framework} for minimizing the objective function in the optimization problem (\ref{overall_objective_function}) are essentially iterative. Here, we investigate their convergence and fastness via experiments. Fig.~\ref{convergence-curves} plots the loss curves of our RNSE on all the selected datasets. In each sub-graph, the $y$-axis denotes the normalized objective function value\footnote{Each iteration's value is divided by the first iteration's value.} and $x$-axis is the iteration number. Obviously, we could read that the RNSE algorithm converges quickly, usually within 20 iterations.

By the way, it's also easy to analyze the computational complexities of the two subproblems (Algorithm~\ref{algo:updatingS} and~\ref{algo:updatingP} ) of RNSE, which corresponds to $O(\mathbf{M}\mathbf{N}^{2}+\mathbf{C}\mathbf{N}^{2})$ and $\mathbf{C}\mathbf{N}^{2}$, respectively (usually $\mathbf{M}, \mathbf{C} \ll \mathbf{N} $). 
Therefore, the RNSE's overall complexity is square to the number of samples which is much faster than many existing competitors' complexity (i.e., $O(\mathbf{N}^3)$). However, for a larger dataset, say million-scale samples, RNSE is reasonably to loss efficiency which posits a big challenge for future work.

\section{Conclusion}
To sum up, the main contributions of this paper are twofold:
First, we prove the advantage of performing spectral-style clustering in an end-to-end single-stage fashion where the similarity matrix is not prefixed but learned adaptively from the data.
Second, we show that the difficult optimization problem of our proposed RNSE technique can be decomposed into two subproblems (i.e., metric projection and orthogonal symmetric non-negative matrix factorization) and then solved by successive alternating projection and strategic multiplicative update respectively.

\bibliographystyle{ieeetran}
\bibliography{refs}

\end{document}